%% file: main.tex
\documentclass{article}
\input{styles}

\usepackage{microtype}
\usepackage{graphicx}
\usepackage{booktabs} 
\usepackage{hyperref}

\usepackage[accepted]{icml2025}

\usepackage{amsmath}
\usepackage{amssymb}
\usepackage{mathtools}
\usepackage{amsthm}

\theoremstyle{plain}
\newtheorem{theorem}{Theorem}[section]
\newtheorem{proposition}[theorem]{Proposition}

\theoremstyle{definition}

\theoremstyle{remark}
\newtheorem{remark}[theorem]{Remark}

\usepackage[textsize=tiny]{todonotes}

\icmltitlerunning{TensorSLM: Energy-efficient Embedding Compression of Sub-billion Parameter Language Models on Low-end Devices}

\begin{document}

\twocolumn[
\icmltitle{\name: Energy-efficient Embedding Compression of Sub-billion Parameter Language Models on Low-end Devices}

\icmlsetsymbol{equal}{*}

\begin{icmlauthorlist}
\icmlauthor{Mingxue Xu}{sch}
\icmlauthor{Yao Lei Xu}{sch}
\icmlauthor{Danilo P. Mandic}{sch}

\end{icmlauthorlist}

\icmlaffiliation{sch}{Imperial College London, London, United Kingdom}
\icmlcorrespondingauthor{Mingxue Xu}{mx1221@ic.ac.uk}

\icmlkeywords{Machine Learning, ICML}

\vskip 0.3in
]

\printAffiliationsAndNotice{} 

\begin{abstract}
\input{sections/abs}
\end{abstract}

\input{sections/intro}
\input{sections/slm}

\input{sections/pre}

\input{sections/methodology}

\input{sections/exp}

\input{sections/conclusion}

\bibliography{ref}
\bibliographystyle{icml2025}

\newpage

\input{sections/appendix/main}

\end{document}

%% file: styles.tex
\usepackage{booktabs}       
\usepackage{amsfonts}       
\usepackage{amssymb}
\usepackage{amsmath}
\usepackage{nicefrac}       
\usepackage{microtype}      
\usepackage[svgnames]{xcolor}         
\usepackage{soul}
\usepackage{multirow}
\usepackage{enumerate}
\usepackage{colortbl}
\usepackage{graphicx}
\usepackage{caption}
\usepackage{subcaption}
\usepackage{xr}
\usepackage{array}
\usepackage{multicol}
\usepackage{natbib}
\usepackage[vlined, ruled, linesnumbered]{algorithm2e}
\usepackage{enumitem}
\usepackage{xr}
\usepackage{array}
\usepackage{float}
\usepackage{hhline}
\usepackage{tikz}
\usetikzlibrary{calc,arrows}

\newcommand{\name}{TensorSLM}

\newcolumntype{?}{!{\vrule width 1pt}}
\newcommand{\tikzmark}[1]{\tikz[remember picture,overlay]\node[coordinate] (#1) {};}

\usepackage{thmtools}
\usepackage{thm-restate}

\usepackage[pagebackref=true,breaklinks=true,colorlinks,bookmarks=false]{hyperref}
\hypersetup{
    linkcolor=Maroon,      
    urlcolor=LightSkyBlue,      
    citecolor=ForestGreen,    
}

\usepackage{amsthm} 
\theoremstyle{plain}     
\theoremstyle{definition}  

\newtheorem{step}{Step}

\usepackage[nameinlink,capitalise]{cleveref}
\crefname{section}{Section}{Sec.}
\crefname{line}{line}{§§}
\crefname{figure}{Figure.}{Figure.}
\crefname{table}{Table.}{Table.}
\crefname{algorithm}{Algorithm}{§§}
\crefname{appendix}{Appx.}{§§}
\crefname{definition}{Def.}{§§}
\crefname{equation}{Eq.}{Eq.}
\crefname{step}{Step.}{§§}
\crefname{remark}{Remark}{§§}
\crefname{paragraph}{Para.}{§§}

\usepackage{tikz}
\newcommand*\circled[1]{\tikz[baseline=(char.base)]{
            \node[shape=circle,draw,inner sep=0.5pt] (char) {#1};}}

%% file: sections/abs.tex
Small Language Models (SLMs, or on-device LMs)~\citep{slm} have significantly fewer parameters than Large Language Models (LLMs). They are typically deployed on low-end devices, like mobile phones~\citep{mobilellm} and single-board computers. Unlike LLMs, which rely on increasing model size for better generalisation, SLMs designed for edge applications are expected to have {\bf adaptivity} to the deployment environments and {\bf energy efficiency} given the device battery life constraints, which are not addressed in datacenter-deployed LLMs.
This paper addresses these two requirements by proposing a training-free token embedding compression approach using Tensor-Train Decomposition (TTD).
Each pre-trained token embedding vector is converted into a lower-dimensional Matrix Product State (MPS).
We comprehensively evaluate the extracted low-rank structures across compression ratio, language task performance, latency, and energy consumption on a typical low-end device, i.e. Raspberry Pi.
Taking the sub-billion parameter versions of GPT-2/Cerebres-GPT and OPT models as examples, our approach achieves a comparable language task performance to the original model with around $2.0\times$ embedding layer compression, while the energy consumption of a single query drops by half. 

%% file: sections/intro.tex
\section{Introduction}\label{sec:intro}
\begin{figure*}[h]
    \centering
    \begin{subfigure}[b]{0.95\textwidth}
        \includegraphics[width=\linewidth]{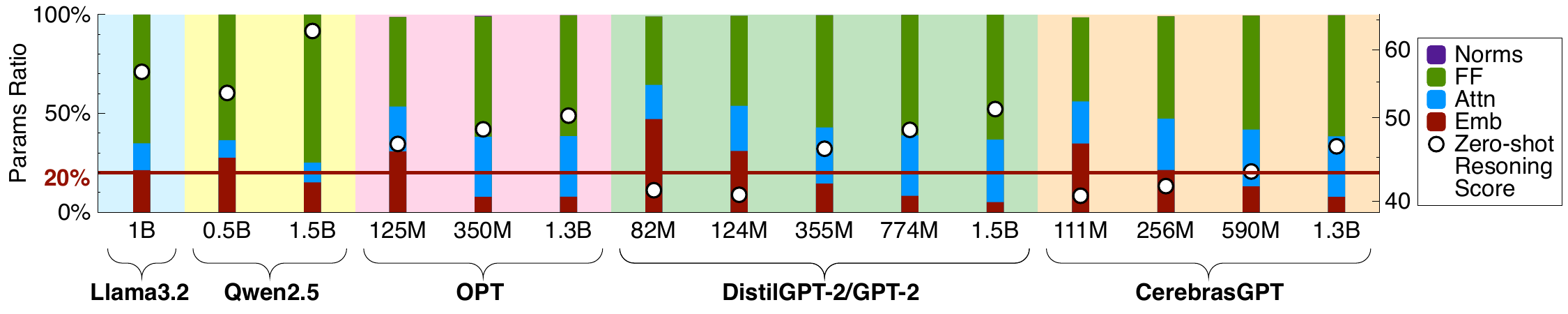}
    \caption{The parameter ratio of Norms (including layer norms), feed-forward layers (FF), attention layers (Attn), and embedding layers (Emb), and the average zero-shot reasoning score~\citep{hellaswag,arc,boolq,piqa} of several open-source model series. In a model series, smaller models have a higher token embedding layer ratio and lower feed-forward layer ratio, while the attention layer ratio is maintained.}\label{fig:layer-ratio}
    \end{subfigure}
    \hfill
    \begin{subfigure}[b]{\textwidth}
    \centering
    \includegraphics[width=\textwidth]{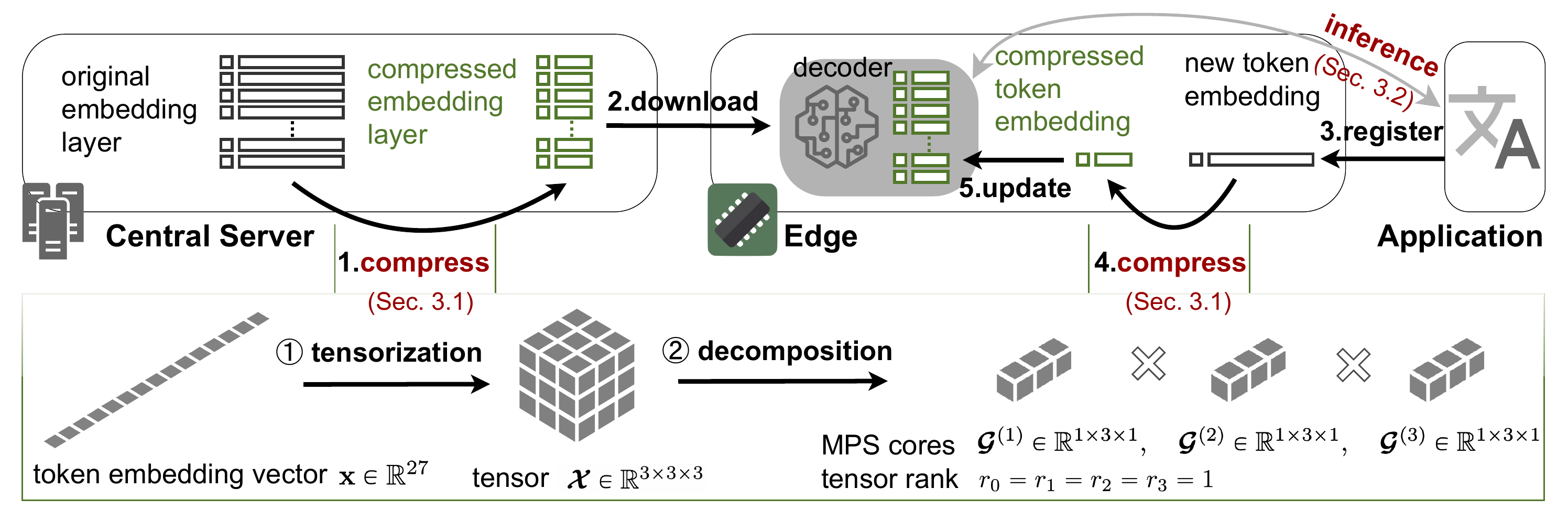}
    \caption{The workflow of SLM compression in edge computing scenario with our approach.}\label{fig:pipeline}   
    \vspace{-5pt}
    \end{subfigure}  
\caption{Typical SLM layer composition and the SLM application requirement of adaptability. }\label{fig:slm-ada}

\end{figure*}

Modelling complex language patterns and solving complex language tasks are two of the primary reasons that Large Language Models (LLMs) have attracted considerable attention in recent years. 
While the LLMs track thrives on increasing model sizes and tackling more difficult tasks, another track is considering putting such capable models on lower-end devices. These models are called Small Language Models (SLMs)~\citep{slm} or on-device language models~\citep{mobilellm,openelm,hfsmollm}.

SLMs may have less than one billion parameters~\citep{openelm,mobilellm,laskaridis2024melting}. Though such a size is already a few tenths or even hundreds of what common LLMs usually are, it can still be burdensome for some low-end devices.
As listed in~\citep[Fig. 2]{mobilellm}, some prevalent mobile devices (e.g. iPhone 14 and iPhone 15) only have 6GB DRAM. 
For some SLMs like Gemma2-2B, running the uncompressed version causes a system crash on Raspberry Pi-5 with 8GB DRAM.

Compared with LLMs, SLMs on low-end devices have different layer compositions of the model and different on-board operations due to the absence of server-level GPUs. As shown in~\cref{fig:layer-ratio}, around half of the investigated open-source models have more than 20\% of the parameters attributed to token embedding layers, which is consistent with the previous findings, i.e.~\citep[Section 2.2.3]{mobilellm}. Additionally, since no server-level GPU is on board to support massive parallel operations for matrix multiplication, block-wise approaches that rely on parallelism~\citep{Dao2022MonarchES,qiu2024compute} are not suitable for low-end deployment scenarios.

To this end, this paper proposes \name, a tensor-based approach to compress SLMs for low-end devices (i.e. Raspberry Pi without GPU). Together with matrix-based low-rank approaches~\citep{chen2018groupreduce,hrinchuk-etal-2020-tensorized,lioutas2020improving,acharya2019online,chen2021drone,hsu2022language,Dao2022MonarchES,qiu2024compute}, this kind of approach forms a broader field named low-rank factorization. The comparison of these works regarding methodologies (e.g. matrix/tensor, with/without training) and applications (e.g. high-end/low-end devices, large/small models) are clarified in~\cref{tab:related}. 

Compared with two-dimensional matrices or their finer-grained block-wise forms~\citep{chen2018groupreduce, Dao2022MonarchES}, higher-order tensors provide more diverse representation alternatives through their inter-order information, which is more suitable for small-size models to model complex patterns. This superiority is more pronounced when no fine-tuning data is available to adjust model parameters for specific deployment environments.

The contributions of this paper are summarised as follows:
\vspace{-5px}
\begin{enumerate}[leftmargin=*]\setlength\itemsep{-0.5em}
    \item We systematically analyse LLMs on high-end GPU servers and SLMs on low-end edge devices to address the two unique requirements of SLM compression: {\it adaptability} to specific deployment environments and {\it energy efficiency} for better user experience.  
    \item To our knowledge, we are the first to \ul{compress SLMs for low-end device use cases using low-rank factorization}. We adjust Tensor-Train Decomposition for non-parallel operations in the forward passes, where block-wise approaches~\citep{Dao2022MonarchES,qiu2024compute} are incompetent. 
    \item We gave the measured latency and estimated energy consumption of SLMs on the typical low-end device, Raspberry Pi 5
    , finding that our approach reduces half of the inference energy with negligible latency increase. 
    \item We evaluated both simple and complex language tasks. We found that our tensor-based approach is better at unprompted and unconstrained question answering than the matrix-based SVD approach, and herein sheds light on selecting appropriate algebraic structures for language model compression according to the specific tasks.  
\end{enumerate}

%% file: sections/slm.tex
\section{Unique Requirements of SLM Applications}\label{sec:slms}

This section clarifies the main application differences between LLMs and SLMs, which will then guide the design of SLMs compression on low-end devices.

\subsection{Adaptability}\label{sec:flex}

Unlike the current LLM applications, which are mostly running on high-end GPU servers (e.g. in the data centres with numerous NVIDIA A100), SLMs are mainly for edge (or mobile) applications that require adapting to the environment with limited resources on lower-end devices. 
A common approach to adapting to the dynamic environment is updating the vocabulary according to the changes in input text distribution~\cite{chen2018groupreduce}. The reasons for this distribution change vary from case to case. For example, new user registration, or the frequently used tokens update with the users' changing daily lives. 

To cope with the ever-changing input tokens and vocabulary, a straightforward strategy is to build a could-edge system, as shown in~\cref{fig:pipeline}, which is similar to the workflows in the field of edge computing, e.g.~\citep[Fig.1]{laskaridis2024melting}. There are two kinds of devices in this workflow: 1) the central server, which is possibly a server in public or private cloud services, or a higher-end personal computer, and 2) the low-end edge device. In this paper, we only talk about a typical edge device - Raspberry Pi. 
Over a fairly long period (e.g. months or years), the central server only communicates with the edge device once to provide a brand-new pre-trained language model.
Afterwards, the edge device should update the vocabulary on board according to the changes in the environment.

A detailed explanation of~\cref{fig:pipeline} is as follows:

\begin{step}
    The central server compresses the whole token embedding matrices on the token embedding level, according to~\cref{ttcp_alg:ttsvd}. 
\end{step}

\begin{step}
    The compressed vocabulary and other parts of the language model (e.g. the decoder) are downloaded and then deployed on a low-end device. 
\end{step}

\begin{step}
    During the application runs, the vocabulary updates for two cases:
    \vspace{-10px}
    \begin{enumerate}\setlength\itemsep{-0.4em}
        \item a new token is required according to the actual application requirements, it will be registered by the service on the edge device. Jump to {\bf Step 4}.
        \item an old token is required to be removed (e.g. it has not been used for a long time), the edge device simply deletes the corresponding token embedding vector. Meanwhile, the application deregisters this token. 
    \end{enumerate}
\end{step}

\begin{step}
    The low-end device compresses the added token embedding vector as described in~\cref{ttcp_alg:ttsvd}.  
\end{step}

\begin{step}
    The current vocabulary of the language model. The compression process of a single token embedding follows a pipeline of \circled{1} tensorization and \circled{2} decomposition.
\end{step}

\subsection{Energy Efficiency}\label{sec:low-energy}

From the workload of the high-end GPU servers (e.g. those equipped with NVIDIA A100) and low-end edge devices (e.g. Raspberry Pi 5) described in~\cref{sec:flex}, we know that the edge device only takes charge of light-weight essential tasks, since it has strict limitations in computation, memory and communication. Furthermore, since battery life directly impacts the user experience, energy consumption is also a significant concern. 

\begin{table}[t]\scriptsize
    \centering
    \caption{Approximate energy consumption of different operations (1nJ=1000pJ). For servers, communication with the wired network (e.g. ethernet or optical fibre) is preferred; for edge devices, it is preferred to use wireless networks (e.g. Wi-Fi or cellular network).}\label{tab:energy}
    
    \begin{tabular}{c|l|c|c}
    \toprule
    \multicolumn{2}{c|}{\textbf{Energy Consumption}} & \begin{tabular}[c]{@{}c@{}}Raspberry Pi 5\\(Cortex-A76 CPU)\end{tabular} & \begin{tabular}[c]{@{}c@{}}GPU server\\(A100 GPU)\end{tabular} \\\midrule
    \multirow{2}{*}{\begin{tabular}{c}
         \textbf{Computation}  \\
         (pJ/\texttt{float32})
    \end{tabular}} &\textbf{Add}  & 1.0-2.5 & 5-12 \\\cline{2-4}
    &\textbf{Mult}  & 1.2-3 & 6-15 \\\midrule
     \multicolumn{2}{c|}{\textbf{Memory} (pJ/\texttt{float32})} & 70-260 & 100-450 \\\midrule
    \multirow{2}{*}{\begin{tabular}{c}
         \textbf{Communication}  \\
         (nJ/\texttt{float32})
    \end{tabular}} &\textbf{Wired}  & \multicolumn{2}{c}{50-350} \\\cline{2-4}
    &\textbf{Wireless}  & \multicolumn{2}{c}{400-6000} \\
    \bottomrule
    
    \end{tabular}

\end{table}

The actual energy consumption of a device depends on various factors, like the semiconductor temperature, system workload, operating environment, etc. Thus, it is hard to precisely calculate the exact energy consumption of an algorithm on a certain hardware. However, we can still estimate the range of energy consumption in the system as~\cref{tab:energy}, where we can have the following remarks:

\begin{remark}\label{re:expensive}
Memory operations are more ``expensive'' than computation in terms of energy.
\end{remark}
\begin{remark}\label{re:avoid}
Non-essential communication should be avoided for energy concerns.
\end{remark}
The workflow in~\cref{fig:pipeline} has already satisfied~\cref{re:avoid}. For~\cref{re:expensive}, if real-time is {\it not} the most important concern in the edge application, we ``exchange'' memory with computation for longer battery life. Further discussion and evaluation around these are in~\cref{sec:med-1,sec:exp}.

%% file: sections/pre.tex
\section{Preliminaries}

This section gives the essential concepts related to tensor, tensor operations and Tensor-Train Decomposition.

\textbf{Order-N Tensor.} An order-$N$ real-valued tensor, $\mathcal{A}$, is a high-dimensional matrix (or multi-way array), denoted by $\mathcal{A}\in\mathbb{R}^{I_1\times\dots\times I_N}$, where $N$ is the order of the tensor (i.e., number of its modes), and $I_k$ ($1 \leq k \leq N$) is the size (i.e., the dimension) of its $k$-{th} mode. In this sense, matrices (denoted as $\mathbf{A}\in\mathbb{R}^{I_1\times I_2}$) can be seen as order-$2$ tensors ($N=2$), vectors (denoted as ${\bf a}\in\mathbb{R}^{I}$) can be seen as order-$1$ tensors ($N=1$), and scalars (denoted as $a\in\mathbb{R}$) are order-$0$ tensors ($N=0$).

\textbf{Tensor-Train Decomposition (TTD).} 
The most common Tensor-Train Decomposition~\citep{oseledets2011tensor} formats a tensor into a Matrix Product State (MPS) form, which applies the Tensor-Train Singular Value Decomposition (TT-SVD) algorithm to an order-$N$ tensor, $\mathcal{X} \in \mathbb{R}^{I_1 \times I_2 \times \cdots \times I_N}$. This results in $N$ smaller $2$-nd or $3$-rd order tensors, $\mathcal{G}^{(k)} \in \mathbb{R}^{ r_{k-1} \times  I_k \times r_k }$ for $k=1, \dots, N$, such that 
\begin{equation}\label{ttcp_eq:tt_def_contract}
    \mathcal{X} \approx \mathcal{G}^{(1)} \times_2^1 \mathcal{G}^{(2)} \times_3^1 \mathcal{G}^{(3)} \times_3^1 \cdots \times_3^1 \mathcal{G}^{(N)}.
\end{equation}
Tensor $\mathcal{G}^{(1)}, \ldots, \mathcal{G}^{(N)}$ are referred to as the {tensor cores}, while the set $\{ r_0, r_1, \ldots, r_{N} \}$ represents the  {TT-rank} of the TT decomposition ($r_0=r_N=1$).

\begin{algorithm}\small
    \SetKwInOut{Input}{Input}
    \SetKwInOut{Output}{Output}
    \SetKwInOut{Initialize}{Initialize}
\caption{\small\texttt{TT\_SVD}\citep{oseledets2011tensor} for a Single Token Embedding Compression }
        \label{ttcp_alg:ttsvd}
        
        \Input{1. $d$-dimensional token embedding vector ${\bf x}\in \mathbb{R}^{d}$, approximation accuracy $\epsilon$;\\
             2. Tensor dimension $\{I_1, I_2, \ldots,I_N\}$ and TT ranks\\ $\{r_0,r_1,\ldots,r_N\}$.
        }
    
        \Output{
            TT cores $\mathcal{G}^{(1)}, \dots, \mathcal{G}^{(N)}$.
        }

        \Initialize{Tensor $\mathcal{X}\leftarrow \texttt{reshape}({\bf x}, [I_1, I_2, \ldots,I_N])$, \\temporary matrix \\ \qquad\qquad$\textbf{Z} \leftarrow \texttt{reshape}(\mathcal{X},[r_0I_1, \prod_{j=2}^{N} I_j])$, \\truncation parameter $\delta = \frac{\epsilon}{\sqrt{N-1}} \Vert\mathcal{X}\Vert_F$.}

        \For{$k=1$ to $N-1$}{

            $\mathbf{U}, \mathbf{S}, \mathbf{V}, \mathbf{E} \leftarrow \texttt{truncSVD}(\textbf{Z},\delta,r_k)$ \label{line:svd}\\
            \qquad\tcp{\small s.t. $\mathbf{U} \in \mathbb{R}^{r_{k-1}I_k \times r_k}$, ${\Vert \mathbf{E} \Vert_F} \leq \delta$} 
    
            $\mathcal{G}^{(k)} \leftarrow$
                \texttt{reshape} {$\left( \mathbf{U}, [ r_{k-1}, I_k, r_k   ] \right)$} \\
                \qquad\tcp{\small get $k$th TT core}\label{line:core}
    
            $\mathbf{Z} \leftarrow$
                \texttt{reshape} {$\left( \mathbf{SV}^T, [r_k I_{k+1}, \prod_{j=k+2}^{N} I_j]) \right)$}\\
                \qquad\tcp{\small $\mathbf{SV}^T \in \mathbb{R}^{\prod^{N}_{i=k+2}I_i}$}\label{line:iterate}
        }
    
        $\mathcal{G}^{(N)} \leftarrow \mathbf{Z}$
    
        \Return{
             $\mathcal{G}^{(1)}, \mathcal{G}^{(2)}, \dots, \mathcal{G}^{(N)}$
        } 
    \end{algorithm}

%% file: sections/methodology.tex
\section{Methodology}\label{sec:methodology}

This section clarifies the technical cornerstones of our approach. A practical pipeline of our approach is depicted in~\cref{fig:pipeline}. The whole vocabulary is processed on higher-end servers, while inference and vocabulary updates happen on lower-end edge devices. 
\subsection{Individual Embedding Vector Compression}\label{sec:med-1}

For the compression of the embedding matrix, rather than decomposing the whole embedding weight matrix, we propose to decompose each embedding vector. The lower half of~\cref{fig:pipeline} is a simplified illustration of such a process, with a detailed description in \cref{ttcp_alg:ttsvd}. 

\textbf{Tensorization.} Each token embedding ${\bf x} \in \mathbb{R}^d$ is reshaped (or folded and tensorized into an order-$N$ tensor.
Denote $\texttt{reshape}(\cdot)$ as the reshape function, $\mathcal{X}=\texttt{reshape}({\bf x}, \{I_1, I_2, \ldots, I_N\})$ and $\mathcal{X} \in \mathbb{R}^{I_1 \times \cdots \times I_N}$ such that $d = \prod_{k=1}^{N} I_k$. In the example in \cref{fig:pipeline}, the token embedding vector ${\bf x}$ is a $27$-dimensional vector, $d=27$. In this way, vector ${\bf x}$ is reshaped into an order-$3$ ($N=3$) tensor $\mathcal{X}$, with tensor size for each mode $I_1=I_2=I_3=3$. 

\textbf{Tensor Decomposition.} Tensor $\mathcal{X}$ is then decomposed and stored in a Matrix Product State (MPS) form as $\mathcal{X} \approx \mathcal{G}^{(1)} \times_3^1 \cdots \times_3^1 \mathcal{G}^{(N)}$, with hyperparameters as TT ranks $r_0, r_1, \ldots, r_N$. For the case in \cref{fig:pipeline}, the MPS cores are $\mathcal{G}^{(1)}$, $\mathcal{G}^{(2)}$, $\mathcal{G}^{(3)}$, with TT ranks $r_0=r_1=r_2=r_3=1$.
In other words, instead of storing the entire token embedding vector $\textbf{x} \in \mathbb{R}^{d}$, we store the corresponding MPS cores, $\mathcal{G}^{(k)} \in \mathbb{R}^{r_{k-1} \times I_k \times r_k}$, for $k=1,\ldots,N$. The parameter count of the MPS cores $\{\mathcal{G}^{(k)}\}$ is $\sum^{N}_{k=1} \vert \mathcal{G}^{(k)} \vert = \sum^{N}_{k=1} r_{k-1}I_k r_{k}$, where $\vert \cdot \vert$ represents the parameter count. 

A more detailed explanation of individual token embedding compression is given in \cref{ttcp_alg:ttsvd}, where $\Vert \cdot \Vert_{F}$ denotes the Frobenius norm. Although the embedding vector is reshaped into a tensor, the decomposition for each mode of this tensor is still based on the matrix-level SVD (\cref{line:svd}). Then the complexity of $\texttt{TT\_SVD}$ can be derived from SVD and its variants, such as truncated SVD~\citep{oseledets2011tensor}. Given the vocabulary size $V$, the original parameters of the embedding layers are compressed from $Vd$ to $V\sum_{k=1}^{N} r_{k-1}I_k r_{k}$, and the compression ratio can be obtained via $\eta_{\texttt{TTD}} =\frac{d}{\sum_{k=1}^{N} r_{k-1}I_k r_{k}} -1$. The computation and memory complexities for all the above processes are summarized in \cref{tab:complexity}.

\textbf{Energy Consumption Analysis.} Recall in~\cref{sec:low-energy} we have~\cref{re:expensive} to guide the choice between memory and computation for the same functionalities from the perspective of energy cost. Based on~\cref{re:expensive} and~\cref{tab:complexity}, we can initially give the estimated energy costs when the SLM processes an input token (only before the decoder), which is similar with~\citep{power-measure}. Assuming in the same operating environment and other conditions (e.g. temperature), the memory energy cost of each \texttt{float32} is $\nu$, and the computation energy cost of each \texttt{float32} is $\tau$, all the model weights are represented in \texttt{float32}. 

When inputting a text of length $l$, denote original model energy cost regarding memory as $\mathcal{E}_{\nu}$, model energy cost regarding computation is $\mathcal{E}_{\tau}$, 
\vspace{-5pt}
\begin{equation}\label{eq:e1}
    \mathcal{E}_{\nu} = \nu (dV + ld), \quad \mathcal{E}_{\tau} = 0,
\end{equation}
and after compression, the energy costs are 
\vspace{-5pt}
\begin{equation}\label{eq:e2}
    \mathcal{E}^{'}_{\nu} = \nu (VNIr^2 + lNIr^2+ld), \quad \mathcal{E}^{'}_{\tau} = \tau NIr^{2}.
\end{equation}
Denote the SVD rank $k$, the energy cost after compressing with matrix-based SVD is
\vspace{-5pt}
\begin{align}\label{eq:e3}
    \mathcal{E}^{''}_{\nu} &= \nu \left [k(V + 2d + l +1) +ld \right ], \\
    \mathcal{E}^{''}_{\tau} &= \tau ( 2ldk-ld+kd).    
\end{align}
Therefore, we have the ratio of inference energy $\omega$, between the compressed language models and the uncompressed models. Denote $\omega_{\texttt{TT}} =\frac{\mathcal{E}^{'}_{\nu} + \mathcal{E}^{'}_{\tau}}{\mathcal{E}_{\nu} + \mathcal{E}_{\tau}}$ as the ratio with \name, and $\omega_{\texttt{SVD}} =\frac{\mathcal{E}^{''}_{\nu} + \mathcal{E}^{''}_{\tau}}{\mathcal{E}_{\nu} + \mathcal{E}_{\tau}}$ as the ratio with SVD. We will give the estimated values of $\omega_{\texttt{TT}}$ and $\omega_{\texttt{SVD}}$ in~\cref{sec:exp-res} according to the hyperparameters of the investigated open-source SLMs.

\input{sections/others/tab-complexity}
\input{figures/res}

\subsection{Language Model Inference Process with the Compressed Embeddings}\label{sec:med-2}

The original inference process with embedding vectors is as follows: when the encoded texts (separated as tokens) are forwarded to the embedding layer, the embedding layer outputs the embedding vectors according to the input tokens; the embedding layer here acts like a look-up table. The embedding vectors are then forwarded to the hidden layers of the transformer, whose size is the same as the dimension of the embedding vectors. Thus, if there is no internal change in the hidden layers, the dimension of the embedding vectors should compile with the dimension of the hidden layers. The compressed embeddings should be reconstructed to the original dimension to enable the forwarding process. This inference happens at the application phase shown in the upper right of~\cref{fig:pipeline}.

Thus just before forwarding embedding vectors to the hidden layers, the memory usage increases from $l\sum_{k=1}^{N} r_{k-1}I_k r_{k}$ to $ld$. However, given that the vocabulary size $V$ is normally much larger than the input token number $l$, that means $V\gg l$. Thus our approach can still significantly reduce the memory usage if the embedding layer takes a significant part of the whole model parameters. The reconstruction process follows the tensor contraction in \cref{ttcp_eq:contraction}, turning the TT cores $\{\mathcal{G}^{(k)}\}$ into a $N$-order tensor $\mathcal{X}$ according to \cref{ttcp_eq:tt_def_contract}, and then vectorizing $\mathcal{X}$ into a full-size embedding vector according to \cref{para:vec}.

%% file: sections/others/tab-complexity.tex
\begin{table}[t]\scriptsize
\caption{Computation and memory complexity during the compression (\cref{sec:med-1}) and inference(\cref{sec:med-2}) of \name. $\mathcal{M}_{\texttt{trans}}$ is the transformer module, $V$ denotes the vocabulary size, $d$ is the original token embedding dimension, and $l$ is the token number of the input text. For simplicity, the dimensions for each mode of the tensor and TT rank are represented as $I$ and $r$.}\label{tab:complexity}
\centering
\begin{tabular}{l|l}
\toprule
 \multicolumn{2}{c}{{\bf Memory}} \\\midrule
Original Embedding Layers & $\mathcal{O}(Vd)$  \\\cline{1-2}
Compressed Embedding Layers & $\mathcal{O}(VNIr^2)$  \\\hline
Compressed Encoded Texts & $\mathcal{O}(lNIr^2)$ \\\cline{1-2}
Intermediate input to $\mathcal{M}_{\texttt{trans}}$ & $\mathcal{O}(ld)$  \\
\bottomrule
\toprule
\multicolumn{2}{c}{{\bf Computation}} \\\midrule
$\texttt{TT-SVD}$ for single token embedding & $\mathcal{O}(NIr^3)$ \\\hline
Reconstruction of single token embedding & $\mathcal{O}(NIr^2)$\\
\bottomrule
\end{tabular}
\end{table}

%% file: figures/res.tex
\begin{figure*}[t]
     \centering
     \begin{subfigure}[b]{0.49\textwidth}
         \centering
         \includegraphics[width=\textwidth]{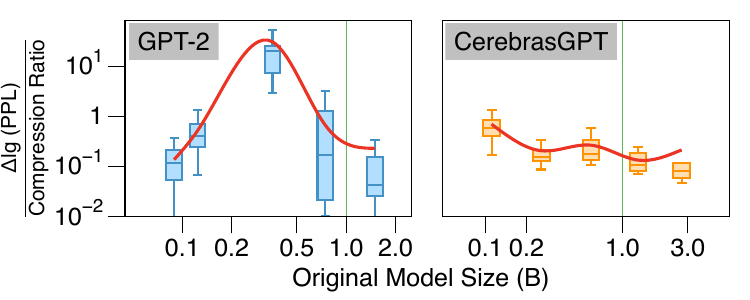}
         \caption{Perplexity-compression trade-off by model size. }
         \label{fig:points}
     \end{subfigure}
     \hfill
    \begin{subfigure}[b]{0.49\textwidth}
         \centering
         \includegraphics[width=\textwidth]{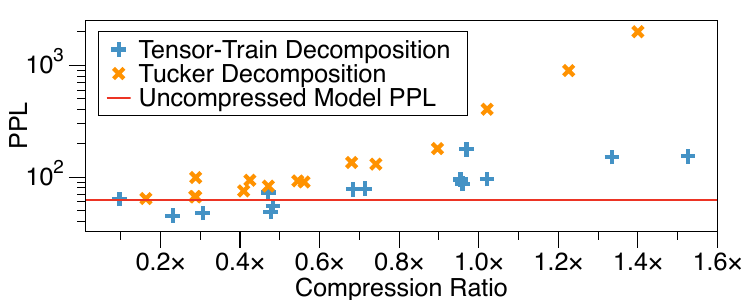}
         \caption{Comparison of different low-rank approaches.}
         \label{fig:tucker}
     \end{subfigure}

         \hfill
    \begin{subfigure}[b]{0.245\textwidth}
         \centering
         \includegraphics[width=\textwidth]{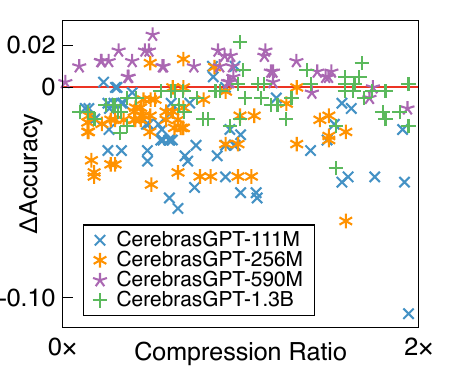}
         \caption{Classification Accuracy}
         \label{fig:cls-acc}
     \end{subfigure}
          \hfill
     \begin{subfigure}[b]{0.245\textwidth}
         \centering
         \includegraphics[width=\textwidth]{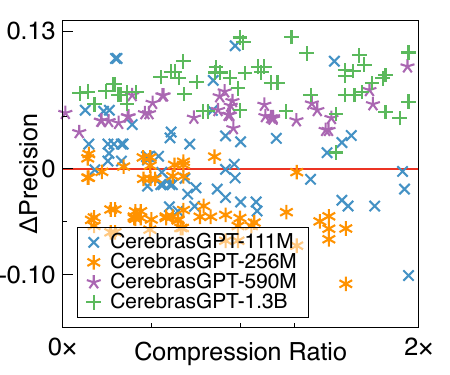}
         \caption{Classification Precision}
         \label{fig:cls-prec}
     \end{subfigure}
     \hfill
          \begin{subfigure}[b]{0.245\textwidth}
         \centering
         \includegraphics[width=\textwidth]{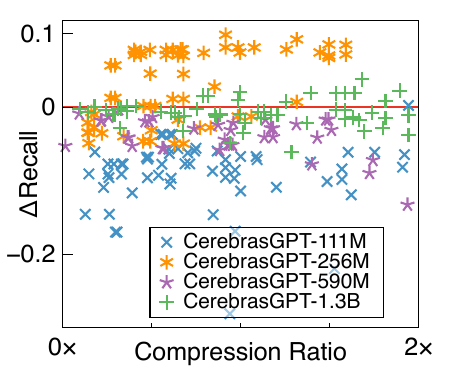}
         \caption{Classification Recall}
         \label{fig:cls-recall}
     \end{subfigure}
     \hfill
    \begin{subfigure}[b]{0.245\textwidth}
         \centering
         \includegraphics[width=\textwidth]{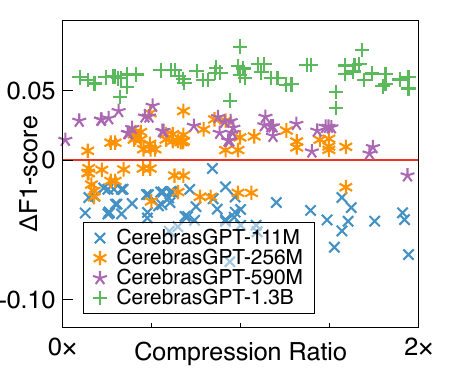}
         \caption{Classification F1-Score}
         \label{fig:cls-f1}
     \end{subfigure}
     \hfill
      \begin{subfigure}[b]{0.49\textwidth}
         \centering
         \includegraphics[width=\textwidth]{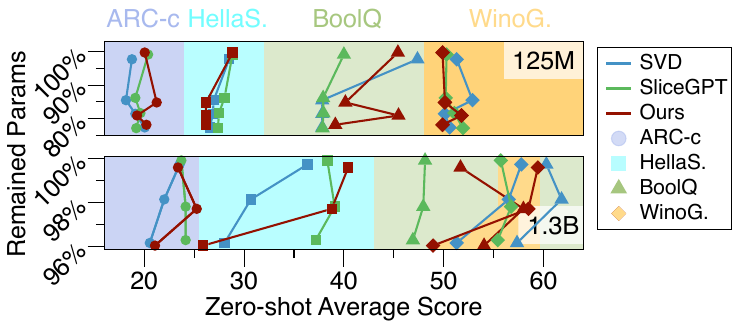}
         \caption{Zero-shot scores of OPT models (125M and 1.3B).}
         \label{fig:zero-shot}
             \vspace{-5pt}
     \end{subfigure}
    \hfill
      \begin{subfigure}[b]{0.49\textwidth}
         \centering
         \includegraphics[width=\textwidth]{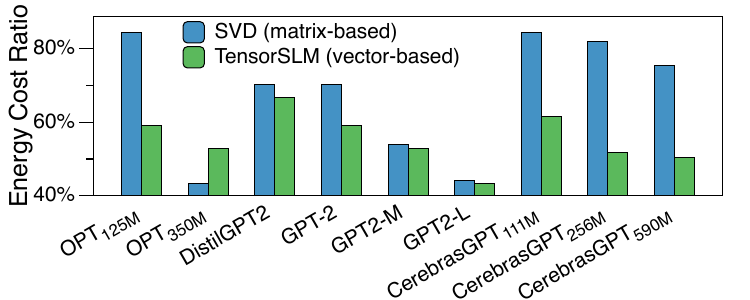}
         \caption{Inference energy costs of different models.}
         \label{fig:energy}
             \vspace{-5pt}
     \end{subfigure}

    \caption{Experimental results. (a): Perplexity-compression trade-off across different model sizes. This trade-off is measured by the ratio between perplexity and compression ratio of embedding layers; lower ratio values indicate better trade-offs. (b): Perplexity of the compressed models with different tensor decomposition approaches. (c)-(f): Task performance on sentiment classification with increasing compression ratio. Higher values indicate better classification performance.  (g): Zero-shot reasoning scores of OPT series models on four different tasks. Our approach demonstrates competitive performance. (h): Energy cost ratio of compressed to uncompressed models, where 100\% represents original energy consumption. Our approach overall outperforms the SVD-based approach.}
    \label{fig:lnppl}
\end{figure*}

%% file: sections/exp.tex
\section{Experimental Evaluation}\label{sec:exp-res}

Our comprehensive experimental evaluation covers compression ratio, language task performance changes, runtime (flops and latency), and energy consumption.

\subsection{Changes of Language Task Performance}\label{sec:lang-task}
\paragraph{Perplexity-compression Trade-off.} In most cases, the shrinkage of model size leads to a drop in the language task performance (though there are exceptions like the accuracy improvement of CerebrasGPT-590M in~\cref{fig:cls-acc}). There should be approaches to measure such a trade-off, with the benefits of a more affordable model size, and how much language task performance has been sacrificed. Here we gave a simple approach for the task evaluated with perplexity, $\frac{\Delta  \lg{\texttt{PPL}(S,\mathcal{M})}}{\eta_{\texttt{emb}}}$, with the measurements on GPT-2 and CerebrasGPT shown in~\cref{fig:points}. We found that larger model sizes achieve better trade-offs, with CerebrasGPT showing a smoother trend compared to GPT-2.

\paragraph{Language Modelling.} 

Due to the combination of tensor size and TT ranks exponentially exploding, we could not test all the possible combinations. However, we can still observe that independent of the tensor orders and the models used for the compression, significant language modelling performance loss tends to appear when the compression ratio exceeds $2.0\times$. We further compared our proposed approach with the Tucker decomposition in~\cref{fig:tucker} with the same tensorization strategy in~\cref{sec:med-1}, and found our adopted Tensor-Train Decomposition outperforms the Tucker Decomposition in perplexity. 

\paragraph{Sentiment Classification.} The results of the sentiment classification task are shown in~\cref{fig:cls-acc,fig:cls-prec,fig:cls-recall,fig:cls-f1}, also indicate that the robustness of larger-scale models (Cerebras-590M and Cerebras-1.3B) is better than that of the smaller models (Cerebras-111M and Cerebras-256M), similar to the trend in language modelling tasks mentioned above. The compressed larger-scale models tend to outperform the original model in precision and F1-score, indicating that our compression improves the ability of the larger models to recognise the positive texts. In contrast, the smaller models tend to have worse performance when the compression ratio increases.

\paragraph{Zero-shot Reasoning.} Since SLMs are incapable of the tasks that are too complex, we only evaluate the relatively simple reasoning tasks (e.g. those that do not involve multi-hop questioning, mathematics or multilingual), and the results are shown in in~\cref{fig:zero-shot}. The bold numbers are the cases that outperform the uncompressed models, or the best in all the compressed cases. 

Our approach has a higher chance of achieving better average reasoning task scores than the SVD-based approach, which implies that our tensors are better at extracting implicit representations in small size models than matrices. Moreover, in our evaluation, our approach generally has higher scores than the SVD-based approach in ARC-challenge and BoolQ. Both of these datasets are more unprompted and unconstrained compared to the other evaluated datasets. This fact implies that our approach may be better at these difficult, unconstrained reasoning tasks.

\subsection{Latency}

While TensorSLM significantly reduces the model parameters and even improves the language tasks performance, in practice it also introduced extra latencies - compression latency (\cref{sec:med-1}) and inference latency(\cref{sec:med-2}). 

In our experimental evaluation, a typically induced latency for an input text was no more than $0.3$ seconds, which is acceptable for edge applications. 
Due to space constraints, the comprehensive results and detailed analysis of the on-device latency evaluation are provided in~\cref{sec:detailed-latency}.

\subsection{Energy Consumption}

The estimated inference energy costs are shown in~\cref{fig:energy}. The Y-axis indicates the ratio between the inference energy costs of the compressed model and that of the uncompressed model; the lower, the better energy saving. 
For each language model, we select the compression case that has a similar language task performance according to~\cref{sec:lang-task}.

We can observe that our approach is mostly better than the SVD-based approach. Furthermore, \name supports adaptivity in edge applications, while the SVD-based approach does not.

%% file: sections/conclusion.tex
\section{Conclusion and Future Work} \label{sec:conclusion}
This paper addresses two unique requirements of Small Language Models (SLMs) deployed on low-end devices: {\it adaptivity} and {\it energy efficiency}. We propose a training-free approach to compress token embeddings using Tensor-Train Decomposition, enabling dynamic vocabulary adjustment and memory-computation trade-offs for extended battery life.
We evaluated our approach on GPT-2, CerebrasGPT, and OPT models across language modeling, classification, and zero-shot reasoning tasks. Systematic measurements on Raspberry Pi 5 show that our method reduces inference energy costs by half, with negligible performance degradation and minimal latency increase.

Future work includes extending tensorization to hidden layers for native compilation and developing accelerated tensor operations to optimize CPU arithmetic requirements.

%% file: sections/appendix/main.tex
\appendix
\onecolumn
\input{sections/appendix/notation}
\input{sections/appendix/preliminaries}
\input{sections/appendix/related}

\input{sections/appendix/logppl}
\input{sections/appendix/proof}
\input{sections/appendix/exp-details}

\input{sections/appendix/other-exp}

%% file: sections/appendix/notation.tex
\section{Notation}
\begin{table}[H]\small
    \centering
    \caption{Notation in this paper.}
    \label{tab:notation}
    \begin{tabular}{r|l}
    \toprule
     {\bf Symbol} & {\bf Meaning} \\
    \midrule
        $a$ & Scalar.\\
         ${\bf x}$& Vector. \\
         ${\bf A}$& Matrix. \\
         $\mathcal{X}$, $\mathcal{A}$, $\mathcal{B}$& Tensor. \\
         $N$ & Tensor order.\\
         $\mathcal{X}[i_1, \ldots, i_N]$ & The $(i_1, i_2,\ldots,i_N)$th entry of the tensor.\\
         $I, I_k$ & Tensor dimension, tensor dimension for the  $k$th mode.   \\
         $\mathcal{M}$ & Model module set. \\
         $\vert \mathcal{M}\vert$, $\vert \mathcal{G}\vert$, $\vert S\vert$ & Parameter count of the model module set $\mathcal{M}$, tensor $\mathcal{G}$ or cardinality of set $S$. \\
         $V$ & Vocabulary of the language model.\\
         $d$ & Token embedding dimension. \\
         $l$ & Input text length.\\
         $r$, $r_k$& TT rank, TT rank of the $k$th mode of the tensor.\\
         $\mathcal{G}^{(k)}$& TT(MPS) core of the $k$th mode of the tensor.\\
         $\times^{p}_{k}$ & Tensor contraction for the $p$th (formal tensor) and $k$th (latter tensor) mode. \\
         $\eta$ & Compression ratio of the entire model.\\
         $\eta_{\texttt{emb}}$ & Compression ratio of the   embedding layer. \\
         $\varphi$ & Parameter reduction ratio of the whole model. \\
         $\varphi_{\texttt{emb}}$ & Parameter reduction   ratio of the embedding layer. \\
         $\nu$ & Memory energy consumption per \texttt{float32} data. \\
         $\tau$ & Computation energy consumption per \texttt{float32} data. \\
         $\mathcal{E}_\nu$ & Estimated energy cost regarding memory. \\
         $\mathcal{E}_\tau$ & Estimated energy cost regarding computation. \\
         $\omega_{\texttt{TT}}$ & Estimated energy cost ratio between the compressed model with \name and uncompressed model.
        \\
        $\omega_{\texttt{SVD}}$ & Estimated energy cost ratio between the compressed model with SVD and the uncompressed model.\\ 
         \bottomrule
    \end{tabular}
    
\end{table}

%% file: sections/appendix/preliminaries.tex
\section{Preliminaries}\label{sec:preliminaries}

\subsection{Tensors and Tensor Operations}
This section gives brief mathematical preliminaries of tensor algebra, and basic knowledge in LLMs to facilitate the understanding of our proposed methodology in \cref{sec:methodology}.

\textbf{Order-N Tensor.} 
An order-$N$ real-valued tensor is a multi-dimensional array, denoted by a calligraphic font, e.g., $\mathcal{A}\in\mathbb{R}^{I_1\times\dots\times I_N}$, where $N$ is the order of the tensor (i.e., number of modes), and $I_n$ ($1 \leq n \leq N$) is the size (i.e., the dimension) of its $n$-{th} mode. Matrices (denoted by bold capital letters, e.g., $\mathbf{A}\in\mathbb{R}^{I_1\times I_2}$) can be seen as order-$2$ tensors ($N=2$), vectors (denoted by bold lower-case letters, e.g., $\mathbf{a}\in\mathbb{R}^{I}$) can be seen as order-1 tensors ($N=1$), and scalars (denoted by lower-case letters, e.g., $a\in\mathbb{R}$) are order-$0$ tensors ($N=0$).
    
\textbf{Tensor Entries.} 
The $(i_1, \ldots, i_N)$-th entry of an order-$N$ tensor is denoted by $a_{i_1, \cdots, i_N} \in \mathbb{R}$, where $i_n = 1, \ldots, I_n$ for $n=1,\ldots,N$. A tensor fiber is a vector of tensor entries obtained by fixing all but one index of the original tensor (e.g., $\mathbf{a}_{:, i_2, i_3, \ldots, i_N} \in \mathbb{R}^{I_1}$). Similarly, a tensor slice is a matrix of tensor entries obtained by fixing all but two indices of the original tensor (e.g., $\mathbf{A}_{:, :, i_3, i_4, \ldots, i_N} \in \mathbb{R}^{I_1 \times I_2}$).

\textbf{Tensorization.} 
A vector ${\bf a} = (a_1, a_2,\ldots,a_{I_1 I_2 \cdots I_N}) \in \mathbb{R}^{I_1 I_2 \cdots I_N}$, can be tensorized (or ``folded'', ``reshaped'') into an order-$N$ tensor $\mathcal{A} \in \mathbb{R}^{I_1 \times I_2 \times \cdots \times I_N}$, so that
\begin{equation}\label{ttcp_eq:folding}
    \mathcal{A}[i_1, i_2, \dots, i_N] = a_{1+\sum_{k=1}^{N}(i_k-1)\prod_{p=1}^{k-1}I_p},\quad\quad 1\leq i_k \leq I_k,
\end{equation}
where $\mathcal{A}[i_1, i_2, \dots, i_N]$ denotes the $(i_1, i_2, \dots, i_N)$-th entry of tensor $\mathcal{A}$.

\textbf{Vectorization.}\label{para:vec}
Given an order-$N$ tensor, $\mathcal{A}\in\mathbb{R}^{I_1\times\cdots\times I_N}$, its vectorization reshapes the high-dimensional matrix into a vector, $\texttt{vec} \left( \mathcal{A} \right) = \mathbf{\bf{a}} \in \mathbb{R}^{I_1 \cdots I_N}$. 

\textbf{Tensor Contraction.} 
The contraction of $\mathcal{A}\in\mathbb{R}^{I_1\times \dots \times I_N}$ and $\mathcal{B}\in\mathbb{R}^{J_1\times \dots \times J_M}$, over the $k$th and $p$th modes respectively, where $I_k=J_p$ is denoted as $\mathcal{A}\times^{p}_{k} \mathcal{B}$ and results in a tensor $\mathcal{C} \in \mathbb{R}^{I_1 \times \cdots \times I_{k-1} \times I_{k+1} \times \cdots \times I_N \times J_1 \times \cdots \times J_{p-1} \times J_{p+1}  \times \cdots \times J_M}$, with entries 
\begin{equation}\label{ttcp_eq:contraction}
\begin{aligned}
    &\mathcal{C}[i_1,\dots,i_{k-1}, i_{k+1}, \dots, i_N, j_1, \dots, j_{p-1}, j_{p+1}, \dots, j_M ] \\ = &\sum_{q=1}^{I_k} \Bigl( \mathcal{A}[i_1, \dots, i_{k-1}, q, i_{k+1}, \dots, i_N] \\
    &\qquad \cdot \mathcal{B}[j_1, \dots, j_{p-1}, q, j_{p+1}, \dots, j_M] \Bigl)
\end{aligned}
\end{equation}

\textbf{Matricization (Mode-n unfolding).} 
Mode-$n$ matricization of a tensor, $\texttt{mat}\left( \mathcal{A}, n \right) = \mathbf{A}_{\{n\}} \in \mathbb{R}^{I_n \times (I_1 \cdots I_{n-1} I_{n+1} \cdots I_N)}$, is a procedure of mapping the elements from a multidimensional array to a two-dimensional array (matrix). Conventionally, such procedure is associated with stacking mode-$n$ fibers (modal vectors) as column vectors of the resulting matrix. For instance, the mode-$1$ unfolding of $\mathcal{A} \in \mathbb{R}^{I_1 \times I_2 \times \cdots \times I_N}$ is represented as $\texttt{mat}\left( \mathcal{A}, 1 \right) = \mathbf{A}_{\{1\}} \in \mathbb{R}^{I_1 \times (I_2 \cdots I_N)}$, where the subscript, $\{1\}$, denotes the mode of matricization, and is given by
    \begin{equation}
        \mathbf{A}_{(1)}\bigg[i_1,\overline{i_2 i_3 \ldots i_N} \bigg] = \mathcal{A}[i_1,i_2,\ldots, i_N]
    \end{equation}
    Note that the overlined subscripts refer to linear indexing (or
    Little-Endian), given by:
    
    \begin{equation}\label{ttcp_eq:mode-n-unfold}
    \begin{aligned}
        \overline{i_1 i_2 \dots i_N}
            &= 1 + \sum_{n=1}^N \Bigg[ (i_n - 1) \prod_{n'=1}^{n-1}I_{n'} \Bigg] \\
            = 1 + i_1 & + (i_2 - 1)I_1 + \cdots + (i_n-1)I_1 \ldots I_{N-1}
    \end{aligned}
    \end{equation}

%% file: sections/appendix/related.tex
\subsection{Related Work in Detail}\label{sec:related}

Low-rank factorization can break the high-dimensional weight matrices into smaller matrices or tensors, so that the overall size of the model can be shrunk. According to the dimensions of the structure that the original weight matrices are broken into, these approaches can be divided into matrix-based and tensor-based.

\textbf{Matrix-based Approaches.}
A straightforward way to shrink the model size is to decompose weight matrices via singular value decomposition (SVD)~\citep{acharya2019online}, which can be further improved by the weighted approach considering the model performance afterwards~\citep{hsu2022language}, knowledge distillation~\citep{lioutas2020improving,mao2020ladabert} and pruning~\citep{mao2020ladabert}. There are also some block-wise decomposition approaches used in language model compression, like Kronecker Products~\citep{tahaei-etal-2022-kroneckerbert,edalati2022kronecker} and data-driven block-wise partitioning~\citep{chen2018groupreduce,chen2021drone}. 

\citep{Dao2022MonarchES,qiu2024compute} used the block-diagonal matrices to reduce the FLOPs in the linear layers computation, with the bonus of shrinking the model size. However, our paper focuses on reducing the parameters of embedding layers, and there is no monotonous relationship between the FLOPs (computation cost) and parameters (memory usage)~\citep{mcunet}. Also, their investigated matrix multiplication only occurs in feed-forward layers, thus their approaches do not fit the embedding layer compression. Moreover, block-diagonal matrices are optimised for GPUs for better parallelization. Our aim of minimizing the number of parameters, makes it optimized for lower-end edge devices rather than GPUs. Indeed, on Raspberry Pi 5, the additional forwarding latency due to compressed embeddings (0.330 - 0.364ms /token in~\cref{tab:decompose-latency}) is even faster than that on GPU (measured as 0.463ms /token in our setting), since there is no parallelization during this forwarding process.

\textbf{Tensor-based Approaches.} Despite some efforts to use tensor decomposition to compress the language model size, all come with an extra training process. The works in \citep{abronin2024tqcompressor} use Kronecker decomposition with row-column permutation during the GPT model fine-tuning process, while \citep{hrinchuk-etal-2020-tensorized} and \citep{chekalina-etal-2023-efficient} propose a tensor-train structured embedding layer and GPT model respectively, yet both train the new-structured model from scratch.

\input{sections/whytt}

%% file: sections/whytt.tex
\section{Why not existing solutions?}\label{sec:whytt}
\begin{table*}[ht]\small
\centering
\caption{Comparison with our approach and the relevant research.}\label{tab:related}

\begin{tabular}{c|cc|c|cc|cc|cc}

\toprule
\multirow{2}{*}{\begin{tabular}{l} 
\textbf{Relevant}\\ \textbf{Study}
\end{tabular}}  & \multicolumn{2}{c|}{\textbf{Device}}   & \multirow{2}{*}{\textbf{Training ?}} & \multicolumn{2}{c|}{ \begin{tabular}{c}
\textbf{Algebra Structure}
\end{tabular}}  &   \multicolumn{2}{c|}{\textbf{Layer}}  &  \multicolumn{2}{c}{\textbf{Focused Size}}    \\ \cline{2-3}\cline{5-10}

 &   \rotatebox[origin=c]{0}{high-end}  &    \rotatebox[origin=c]{0}{low-end}    &   & \rotatebox[origin=c]{0}{matrix} &   \rotatebox[origin=c]{0}{tensor} & \rotatebox[origin=c]{0}{Emb} &  \rotatebox[origin=c]{0}{Linear} & \rotatebox[origin=c]{0}{large} &   \rotatebox[origin=c]{0}{small}
\\\midrule
\cite{chen2018groupreduce}  & $\surd$    &   & $\surd$ &   $\surd$ &   &   $\surd$ & && $\surd$  \\ \hline
\cite{hrinchuk-etal-2020-tensorized}   & $\surd$   &  &  &  & $\surd$  &  $\surd$ & && $\surd$   \\ \hline
\cite{lighttoken}  & $\surd$    &   & $\surd$ &   $\surd$ & &  $\surd$ & & & $\surd$ \\ \hline
\cite{dsvd} & &  $\surd$ &$\surd$ &    $\surd$ & &  $\surd$ & &  & $\surd$     \\ \hline
\cite{irvq}   &  \multicolumn{2}{c|}{-}         & $\surd$  &  $\surd$  &    & \multicolumn{4}{c}{-}                     \\ \hline
\cite{dcq} & $\surd$    &   & $\surd$ &   $\surd$ &   &   $\surd$ & && $\surd$  \\ \hline
\cite{asvd}  & $\surd$     &     & &  $\surd$   &     &  & $\surd$  & $\surd$ & \\ \hline
\cite{hsu2022language}  & $\surd$    &   & $\surd$ &  $\surd$ &    & &  $\surd$ & & $\surd$  \\ \hline
\cite{chekalina2023efficient} & $\surd$   &   & $\surd$  &     & $\surd$  &  & $\surd$ & $\surd$ \\ \hline
\cite{modegpt} & $\surd$    &   & &  $\surd$ &    & &  $\surd$ & $\surd$ &   \\ \hline
\cite{Dao2022MonarchES}   & $\surd$    &   & $\surd$ &    $\surd$    &        &  &  $\surd$     &   &             \\ \hline
\cite{qiu2024compute}    & $\surd$    &   & $\surd$ &    $\surd$    &    $\surd$     &  &  $\surd$     &   &               \\\hline
\cite{mobilellm}  & $\surd$    &   & $\surd$ &     \multicolumn{4}{c|}{-}       &  &  $\surd$        \\\hline
\name (Ours)   &   &  $\surd$     & & &$\surd$ &    $\surd$ &           & $\surd$ \\
\bottomrule

\end{tabular}

\end{table*}

The field of language model compression with low-rank factorization has been booming in recent years. The recent relevant works are summarized in~\cref{tab:related}.
We can observe that for the current existing works, some are specialized for embedding layers~\citep{chen2018groupreduce,hrinchuk-etal-2020-tensorized,lighttoken,dsvd,acharya2019online,irvq} while others are not~\citep{chekalina-etal-2023-efficient,chen2021drone,hsu2022language,Dao2022MonarchES,qiu2024compute}. However, all of these require an extra training process, such as fine-tuning, meta-learning~\citep{chen2018groupreduce,chen2021drone,hsu2022language,dsvd,irvq,Dao2022MonarchES,lighttoken,qiu2024compute} and training from scratch~\citep{hrinchuk-etal-2020-tensorized,chekalina-etal-2023-efficient}. 

There are two limitations to this extra training: 1) extra training involves additional computation and training data, which may be unavailable for low-end devices; 2) training the language model from scratch discards the valuable knowledge stored in the weights of the original models. However, we only focus on training-free low-end device applications. For a more detailed discussion of these relevant works, please refer to~\cref{sec:related}.

%% file: sections/appendix/logppl.tex
\section{Perplexity and Logarithmic Perplexity.}\label{appd:ppl} 

Perplexity is used as a performance evaluation metric of the language modelling task, which has the following form
\begin{equation}\label{eq:ppl}
    \texttt{PPL}(S,\mathcal{M}) = \left( \prod^{|S|}_{i=1} p_{\mathcal{M}}(x_i | x_1, x_2, \ldots, x_{i-1}) \right)^{-1}
\end{equation}
where $S$ is an ordered set (token sequence), consisting of a set of tokens $\{x_t\}$, $t=1,2,\ldots, |S|$, and $\mathcal{M}$ is the model block that contains all the modules of the language model we evaluate. 

Notice that the compression ratio \cref{def:cr} has a linear form, while perplexity \cref{eq:ppl} has an exponential form, so it is hard to combine them as a description of a model compression result, since when compression ratio $\eta$ linearly increases, the perplexity $\texttt{PPL}$ explodes exponentially. To this end, we use the following logarithmic form to describe the language modelling performance
\begin{equation}\label{eq:ppl-ln}
    \ln{\texttt{PPL}(S,\mathcal{M})} = - \sum ^{|S|}_{i=1} \ln p_{\mathcal{M}}(x_i | x_1, x_2, \ldots, x_{i-1})
\end{equation}

Now, the language modelling performance change before and after compression is given by  
\begin{align}\label{def:logppl}
    \Delta & \ln{\texttt{PPL}(S,\mathcal{M})} = \ln{\texttt{PPL}(S,\mathcal{M_{\texttt{cmpr}}})} - \ln{\texttt{PPL}(S,\mathcal{M}_0)} \\  & = \sum ^{|S|}_{i=1} \ln \frac{p_{\mathcal{M}_{0}}(x_i | x_1, x_2, \ldots, x_{i-1})}{p_{\mathcal{M}_{\texttt{cmpr}}}(x_i | x_1, x_2, \ldots, x_{i-1})},
\end{align}\small
observe that \cref{def:logppl} exhibits linearity, like \cref{def:cr}. 

%% file: sections/appendix/proof.tex
\section{Proof of the Highest Compression Ratio in~\cref{tab:complexity}}\label{app:2-power}
\begin{proposition}\label{pro:2}
For an order-$N$ tensor whose dimension for each order are $I$, its TT-format yields the highest compression ratio when $I=2$ and TT rank $r=1$. 
\end{proposition}

\begin{proof}
Assume the tensor size $\left[I_1,\ldots,I_{N}\right]$ for the tensor $\mathcal{X}$ to achieve the highest compression rate, we next give the proof of this hyperparameter selection.

The compression ratio in~\cref{sec:med-1} can be represented as
\begin{align}
 \eta &= \frac{V\times d} 
 {\sum_{j=1}^{V}\sum_{n=1}^{N}(r_{n-1} \times I_n \times r_n)_j} \\
 &= \frac{V \times d}{I_{1}r_{1}+r_{1}I_{2}r_{2}+\cdots+r_{N-2}I_{N-1}r_{N-1}+r_{N-1}I_{N}} \\
 & = \frac{V\times d}{\sum_{k=1}^{\lfloor \frac{N+1}{2} \rfloor} r_{2k-1} \left( r_{2k-2}I_{2k-1}+I_{2k}r_{2k+1} \right)}
\end{align}

For the simplest case, assume $I_1=\cdots=I_{N}=I$ and $r_1=\cdots=r_N=r$. Given $d=\prod_{n=1}^{N} I_n=I^{N}$, we have $N=\log_{I}{D}$, and
\begin{equation}\label{eq:com-app}
    \eta = \frac{V\times d}{rI\left[2+(N-2)r\right]}=\frac{V\times d}{rI\left[2+(\log_{I} d-2)\right]}.
\end{equation}

In Equation~\cref{eq:com-app}, the numerator is a constant, and in the denominator, $R$ is a hyperparameter for the Tensor-Train Decomposition. 
Thus the objective function for the highest compression rate $\eta$ is
\begin{align}
    \min_{I, N} rI\left[2+(N-2)\right] \quad \quad \textbf{s.t.} \quad  
    &N=\log_{I}{d} \label{eq:com-obj} \\ 
    & I, N, r \in \mathbb{Z}^{+} \\
    & 2\leq I \leq N \leq \lfloor\log_2 d\rfloor  \label{eq:com-range}
\end{align}

Regarding \cref{eq:com-obj}, if eliminate $N$ then we have a function $h=rI\left[2+(\log_{I}{d}-2)\right]$. Regarding $d$ in \cref{eq:com-range}, the largest token embedding size of recent GPT-3~\citep{brown2020language} is 12,888. Thus, for the GPT series models no later than GPT-3, \cref{eq:com-obj} should be  $2\leq I \leq N \leq 13$. In this range, $h$ is a monotonically increasing function, where the minimum $h$ occurs at $I=2$. 

Therefore, for the simplest case, we have the best hyperparameter selection of $I_1=I_2=\cdots=I_N=2$, and $N=\lfloor\log_2 d\rfloor$.

\end{proof}

%% file: sections/appendix/exp-details.tex
\section{Experimental Setup}\label{sec:exp}

\subsection{Models, Tasks and Dataset.} 
The sub-billion models we used are DistilGPT2~\citep{sanh2019distilbert}, GPT2, GPT2-{M/L}~\citep{radford2019language}, CerebrasGPT-{111M/256M/590M}~\citep{dey2023cerebras}, OPT-{125M. We also tested the models of slightly over a billion parameters for language task performance with GPT2-XL (1.5 billion parameters), CerebrasGPT-1.3B and OPT-1.3B for the boundary tests. 

Regarding the language tasks, we have two different level language tasks:
\begin{itemize}[leftmargin=2em, itemindent=0em]

    \item {\bf Simple Tasks}: language modelling and sentiment classification. For language modelling, the considered datasets are WikiText2, WikiText103~\citep{merity2022pointer} and 1BW~\citep{chelba2014one}. For sentiment classification, the considered dataset is IMDB~\citep{maas-EtAl:2011:ACL-HLT2011}.
    \item {\bf Complex Tasks}: zero-shot common sense reasoning tasks. The common sense reasoning datasets include ARC-challenge~\cite{arc}, BoolQ~\cite{boolq}, HellaSwag~\cite{hellaswag} and WinoGrade~\cite{sakaguchi2021winogrande}.    
    
\end{itemize}

\subsection{Hardware.} 
Our main experiments were completed on a GPU workstation with an RTX A6000 48GB GPU and AMD Ryzen Threadripper PRO 5955WX CPU. The GPU resource was mainly used to fine-tune language modelling models for sequence classification, which is the requirement of the sentiment classification task. The inference latency of the low-end devices was measured on a Raspberry Pi 5, with a 64-bit Arm Cortex-A76 CPU and 8GB DRAM. The power meter we used is YOJOCK J7-c USB C Tester USB Power Meter, with a single refresh time of more than 500ms. 

\subsection{Evaluation Metrics}\label{sec:metrics}

\textbf{Compression Ratio.} Denote $\mathcal{M}$ as a model block set containing a list of model modules like embedding layers and attention layers. With $\mathcal{M}_0$ as the original model block set, $\mathcal{M}_\texttt{cmpr}$ as the compressed version of $\mathcal{M}_0$, and $\vert\mathcal{M}\vert$ as the parameter count of $\mathcal{M}$. The compression ratio $\eta$ is defined as  
\begin{equation}\label{def:cr}
  \eta = \frac{\vert \mathcal{M}_0 \vert - \vert\mathcal{M}_\texttt{cmpr}\vert}{\vert \mathcal{M}_\texttt{cmpr} \vert}. 
\end{equation}
Specifically, the embedding compression rate is $\eta_{\texttt{emb}} = \frac{\vert \mathcal{T}_0 \vert - \vert\mathcal{T}_\texttt{cmpr}\vert}{\vert \mathcal{T}_0 \vert}$, where $\mathcal{T}$ only contains token embedding layer and position embedding layer. 

\textbf{Perplexity and Logarithmic Perplexity.} We use perplexity ($\texttt{PPL}$) as our metrics of language modelling. Furthermore, we use the logarithmic form of perplexity ($\ln{\texttt{PPL}}$ ) and its change ($\Delta \ln{\texttt{PPL}}$) to align with the linearity of the compression ratio \cref{def:cr}, as defined in~\cref{eq:ppl}.

\textbf{Accuracy, Precision, Recall and F1-Score.} We use these four common evaluation metrics for classification to analyze the classification performance of the compressed model comprehensively. To investigate the performance change before and after compression, we use the difference between the metric values after and before the compression.

\textbf{Zero-shot Reasoning Scores.} For the metrics of reasoning tasks, we use the scores from~\citep{arc,boolq,hellaswag,piqa,sap2019social,sakaguchi2021winogrande}.

\textbf{Energy Consumption.}  Since the actual energy consumption depends on multiple uncontrollable factors, as we discussed in~\cref{sec:low-energy}, it is difficult to isolate compression energy cost from the actual measurements. Thus, we use similar approaches in~\citep{luo2024addition} to estimate the energy consumption.

We use the notations in~\cref{tab:complexity} and~\cref{eq:e1,eq:e2,eq:e3}, and approximate the ratio between computation energy cost and memory energy cost per $\texttt{fload32}$ data as $\frac{\nu}{\tau}=5$. 
Then, we got the configurations of the current open-source SLMs for the values of $d$, $V$ in ~\cref{eq:e1,eq:e2,eq:e3}.
Though we cannot get the actual energy costs, we can compare the inference energy costs of compressed and uncompressed models with this approach.

%% file: sections/appendix/other-exp.tex
\section{On-device Latency Explained with Experimental Results}\label{sec:detailed-latency}

For the compression latency, we investigated the compression latency on the token level, as shown in~\cref{tab:decompose-latency}. Here, ``original'' means the uncompressed model, while $\texttt{PPL}_{\alpha}$ means the compressed model with a negligible task performance drop. In our case, ``negligible task performance drop'' means in the language modelling task, the perplexity is no more than $100.0$. The notation $\varphi_{\text{max}}$ refers to the compressed model with maximum compression ratio.
We observed that for individual token embeddings, there was no significant latency difference between high-end servers and Raspberry Pi, typically no more than 2 milliseconds for each token. Thus, it is acceptable for the Raspberry Pi to compress the individual token embeddings.

\begin{table}[h]\scriptsize
\centering
\caption{The latency (ms/token) of tensorization \& decomposition token embedding vectors and reconstruction on the high-end and lower-end devices. $\texttt{PPL}_{\alpha}$ means the compressed model with a negligible task performance drop, and the symbol $\varphi_{\text{max}}$ represents the case with a maximum compression ratio. $d_{\text{emb}}$ is the embedding dimension of the token embedding vector, and the tested models are GPT-2 and GPT-2-M. On the CPU level, for single token embedding vector decomposition and reconstruction, both server and edge devices have no significant computation overhead.}\label{tab:decompose-latency}

\begin{tabular}{l|l|cc|ll}
\toprule
\multirow{2}{*}{\begin{tabular}{c}
\textbf{Device (CPU)} \\
(ms/token)
\end{tabular}}              & \multirow{2}{*}{$d_{\text{emb}}$} & \multicolumn{2}{c|}{{\begin{tabular}[c]{@{}l@{}}\textbf{tensorization}\\ \& \textbf{decomposition}\end{tabular}}}             & \multicolumn{2}{c}{\textbf{reconstruction}}          \\ \cline{3-6} 
     &                   & \multicolumn{1}{c|}{$\texttt{PPL}_{\alpha}$} & \multicolumn{1}{c|}{$\varphi_{\text{max}}$} & \multicolumn{1}{c|}{$\texttt{PPL}_{\alpha}$} & \multicolumn{1}{c}{$\varphi_{\text{max}}$}  \\ \midrule
\multirow{2}{*}{\textbf{Server}}         &        768           & \multicolumn{1}{l|}{0.627}    &    1.429                          & \multicolumn{1}{l|}{0.117}    &     0.238                   \\ \cline{2-6} 
     &        1024           & \multicolumn{1}{l|}{0.452}    &      1.512                        & \multicolumn{1}{l|}{0.114}    &             0.261             \\ \midrule
\multirow{2}{*}{\textbf{Raspberry Pi 5}} &        768           & \multicolumn{1}{l|}{0.760}    &    1.948                            & \multicolumn{1}{l|}{0.330}    &            0.468       \\ \cline{2-6} 
     &        1024           & \multicolumn{1}{l|}{0.612}    &    2.148                          & \multicolumn{1}{l|}{0.364}    &   0.614                        \\ \bottomrule
\end{tabular}
\end{table}

For the inference latency of a single text, we chose a typical text length of 50 tokens, as shown in~\cref{tab:exp-profile}. we used ``original'', $\texttt{PPL}_{\alpha}$, $\varphi_{\text{max}}$ same as those in~\cref{tab:decompose-latency}, to represent the uncompressed model, the compressed model with a negligible task performance drop and the model with a maximum compression ratio. A typical induced latency for an input text was no more than $0.3$ seconds, which is acceptable for edge applications. 

It should be noted that the embedding reconstruction latency depends on both tensor shapes and flops, and the on-device memory management varies when models of different sizes are loaded. Consequently, in~\cref{tab:exp-profile}, flops alone does not provide a complete predictor of on-device inference latency.

{\bf The cases of} $\varphi_{\text{max}}$ {\bf are typically slower than the cases of} $\texttt{PPL}_{\alpha}$. We have demonstrated in~\cref{app:2-power} that the $\varphi_{\text{max}}$ (the maximum compression ratio) corresponds to cases where embedding vectors are compressed into TT-formatted tensors of the highest orders. During the forward passes, the TT-format of these tensors is decompressed order by order. For example, for an $ N$-order TT-formatted tensor, the decompression process involves $(N-1)$ {\it serial} matrix multiplications. 

This implies that the higher the tensor order, the more matrix multiplication rounds are executed, potentially resulting in slower decompression. A (not representative) exception is the compression for CerebrasGPT-256M, which has the $5$-order tensor shape $4\times2\times17\times4\times2$ for $\texttt{PPL}_{\alpha}$, and $7$-order tensor shape $2\times2\times2\times2\times17\times2\times2$ for $\varphi_{\text{max}}$. The decompression process for each embedding vector involves $4$ and $6$ matrix multiplications respectively, which differ by only two matrix multiplications — a relatively small gap compared to other cases  (e.g. $3$-order for $\texttt{PPL}_{\alpha}$ and $10$-order for $\varphi_{\text{max}}$). Thus, for CerebrasGPT-256M, the on-device inference latencies of $\texttt{PPL}_{\alpha}$ and $\varphi_{\text{max}}$ are similar, as shown in~\cref{tab:exp-profile}. 

Though this compression approach does not provide latency reduction benefits, it does offer advantages in the reduction of memory usage and energy consumption.

{\bf DistilGPT2 exhibits different flops-latency trends from the others}.  In~\cref{tab:exp-profile} the compression for DistilGPT2 has significantly less latency for $\varphi_{\text{max}}$ than $\texttt{PPL}_{\alpha}$, which contradicts the analysis in the preceding paragraph. A possible reason is the different memory scheduling processes of embedding layers and non-embedding layers. DistilGPT has the same embedding layer weight matrix size ($50257 \times 768$) as GPT-2, yet has significantly fewer non-embedding layer parameters (and hence $\sim 50\%$ fewer non-embedding memory pages during inference). This difference may lead to
distinct memory management dynamics.

\input{sections/others/tab-hardware}

%% file: sections/others/tab-hardware.tex
\begin{table*}[t]\scriptsize
    \caption{Parameters, number of floating-point operations (flops) of the compressed and uncompressed sub-billion models, and latency on Raspberry Pi CPU. For flops, the token number of the input texts is $100$, while for latency on Raspberry Pi, the token number is $50$.}\label{tab:exp-profile}
    \centering
\begin{tabular}{ll?p{4cm}|l|l|l?l|l|l}
\toprule
\multicolumn{2}{c?}{\multirow{2}{*}{\textbf{GPT Models}}}                                & \multicolumn{4}{c?}{\textbf{GPT2}}                                                                          & \multicolumn{3}{c}{\textbf{CerebrasGPT}}                              \\ \cline{3-9} \addlinespace[3pt] 
\multicolumn{2}{l?}{}                                                    & \multicolumn{1}{c|}{{\bf DistilGPT2}} & \multicolumn{1}{c|}{{\bf GPT-2}} & \multicolumn{1}{c|}{{\bf GPT-2-M}} & \multicolumn{1}{c?}{{\bf GPT-2-L}} & \multicolumn{1}{l|}{{\bf 111M}} & \multicolumn{1}{l|}{{\bf 256M}} & {\bf 590M} \\ \midrule
\multicolumn{1}{l?}{\multirow{3}{*}{\begin{tabular}[c]{@{}l@{}}\textbf{\# Params}\\ (M)\end{tabular}}} & original               & \multicolumn{1}{l|}{$81.9$}           & \multicolumn{1}{l|}{$124.44$}     & \multicolumn{1}{l|}{$354.82$}       &  $774.03$      &  \multicolumn{1}{l|}{$111.05$}     & \multicolumn{1}{l|}{$255.98$}     &   $590.31$   \\ \cline{2-9} 
\multicolumn{1}{l?}{}                           & $\texttt{PPL}_{\alpha}$                   & \multicolumn{1}{l|}{$67.06$}           & \multicolumn{1}{l|}{$106.36$}     & \multicolumn{1}{l|}{$326.45$}       &   $734.28$     & \multicolumn{1}{l|}{$101.78$}     & \multicolumn{1}{l|}{$226.69$}     &   $543.45$   \\ 
\multicolumn{1}{l?}{}                           & $\varphi_{\text{max}}$  & \multicolumn{1}{l|}{$43.45$}           & \multicolumn{1}{l|}{$85.99$}     & \multicolumn{1}{l|}{$303.88$}       &    $710.83$    & \multicolumn{1}{l|}{$71.87$}     & \multicolumn{1}{l|}{$200.59$}     &    $511.07$  \\ \midrule
\multicolumn{1}{l?}{\multirow{3}{*}{\begin{tabular}[c]{@{}l@{}}\textbf{flops} \\ ($10^6$/text )\end{tabular}}}     & original               & \multicolumn{1}{l|}{$20250$}           & \multicolumn{1}{l|}{$40490$}     & \multicolumn{1}{l|}{$142250$}       &    $330980$    & \multicolumn{1}{l|}{$14470$}     & \multicolumn{1}{l|}{$40400$}     &    $103060$  \\ \cline{2-9} 
\multicolumn{1}{l?}{}                           & $\texttt{PPL}_{\alpha}$                    & \multicolumn{1}{l|}{\tikzmark{distilstart1}$+ 1.65$}           & \multicolumn{1}{l|}{$+ 1.88$}     & \multicolumn{1}{l|}{$+ 3.11$ }       &  \multicolumn{1}{l?}{\tikzmark{Lstart1} $+ 2.30$}      & \multicolumn{1}{l|}{$+ 0.38$}     & \multicolumn{1}{l|}{$+ 1.63$}     &  $+ 2.30$   \\ 
\multicolumn{1}{l?}{}                           & $\varphi_{\text{max}}$  & \multicolumn{1}{l|}{$+ 0.13$}           & \multicolumn{1}{l|}{\tikzmark{distilend1}$+ 0.13$}     & \multicolumn{1}{l|}{$+ 0.20$}       &    $+ 0.25$    & \multicolumn{1}{l|}{\tikzmark{Lend1}$+ 0.13$}     & \multicolumn{1}{l|}{$+ 0.12$}     &   $+ 0.26$   \\ \midrule
\multicolumn{1}{l?}{\multirow{3}{*}{\begin{tabular}[c]{@{}l@{}}\textbf{Latency on} \\ \textbf{Raspberry} \\ \textbf{Pi} (s/text)\end{tabular}}}   & original               & \multicolumn{1}{l|}{$0.19_{\pm 0.02}$}           & \multicolumn{1}{l|}{$0.50_{\pm 0.19}$}     & \multicolumn{1}{l|}{$1.23_{\pm 0.12}$}       &   \multicolumn{1}{l?}{ $3.01_{\pm 0.47}$}    & \multicolumn{1}{l|}{$0.47_{\pm 0.21}$}     & \multicolumn{1}{l|}{$0.71_{\pm 0.02}$}     &   $1.81_{\pm 0.25}$   \\ \cline{2-9} 
\multicolumn{1}{l?}{}                           & $\texttt{PPL}_{\alpha}$                     &  \multicolumn{1}{l|}{\tikzmark{distilstart2} $0.36_{\pm 0.19}$}          & \multicolumn{1}{l|}{$0.50_{\pm 0.16}$}     & \multicolumn{1}{l|}{$1.26_{\pm 0.22}$}       &  \multicolumn{1}{l?}{\tikzmark{Lstart2} $3.01_{\pm 0.29}$}     & \multicolumn{1}{l|}{$0.48_{\pm 0.23}$}     & \multicolumn{1}{l|}{$1.01_{\pm 0.29}$}     &    $1.89_{\pm 0.28}$  \\ 
\multicolumn{1}{l?}{}                           & $\varphi_{\text{max}}$  & \multicolumn{1}{l|}{$0.19_{\pm 0.03}$}        &    \multicolumn{1}{l|}{\tikzmark{distilend2}$0.71_{\pm 0.38}$}     & \multicolumn{1}{l|}{$1.55_{\pm 0.36}$}       &    $3.52_{\pm 0.44}$    & \multicolumn{1}{l|}{\tikzmark{Lend2}$0.72_{\pm 0.42}$}     & \multicolumn{1}{l|}{$0.95_{\pm 0.27}$}     &   $1.91_{\pm 0.24}$   \\ \bottomrule
\end{tabular}
\begin{tikzpicture}[remember picture, overlay]
\def\colwidth{1.6}
\def\rowheight{3}

\node[blue] at ($(distilend2) + (-0.6,-0.3)$) {$-47\%$};
\node[red] at ($(distilend2) + (-0.6+\colwidth,-0.3)$) {$+42\%$};
\node[red] at ($(distilend2) + (-0.6+2*\colwidth,-0.3)$) {$+23\%$};

\node[red] at ($(distilend1) + (-0.5,0.2)$) {\tiny$-92\%$};
\node[red] at ($(distilend1) + (-0.5+\colwidth,0.2)$) {\tiny$-93\%$};
\node[red] at ($(distilend1) + (-0.5+2*\colwidth,0.2)$) {\tiny$-94\%$};
\node[red] at ($(distilend1) + (-0.5+3*\colwidth,0.2)$) {\tiny$-89\%$};

\node[red] at ($(Lend2) + (-0.6,-0.3)$) {$+16\%$};
\node[red] at ($(Lend2) + (-0.6+\colwidth,-0.3)$) {$+50\%$};
\node[blue] at ($(Lend2) + (-0.6+2*\colwidth,-0.3)$) {$-6\%$};
\node[red] at ($(Lend2) + (-0.6+3*\colwidth,-0.3)$) {$+1\%$};

\node[red] at ($(distilend1) + (-0.5+4*\colwidth,0.2)$) {\tiny$-66\%$};
\node[red] at ($(distilend1) + (-0.5+5*\colwidth,0.2)$) {\tiny$-93\%$};
\node[red] at ($(distilend1) + (-0.5+6*\colwidth,0.2)$) {\tiny$-84\%$};

\draw[->, blue] ($(distilend2) + (-0.4,0.35)$) to[bend left=45] ($(distilend2) + (-0.4,-0.05)$);
\draw[->, red] ($(distilend2) + (-0.4+\colwidth,0.35)$) to[bend left=45] ($(distilend2) + (-0.4+\colwidth,-0.05)$);
\draw[->, red] ($(distilend2) + (-0.4+2*\colwidth,0.35)$) to[bend left=45] ($(distilend2) + (-0.4+2*\colwidth,-0.05)$);
\draw[->, red] ($(distilend2) + (-0.4+3*\colwidth,0.35)$) to[bend left=45] ($(distilend2) + (-0.4+3*\colwidth,-0.05)$);
\draw[->, red] ($(distilend2) + (-0.4+4*\colwidth,0.35)$) to[bend left=45] ($(distilend2) + (-0.4+4*\colwidth,-0.05)$);
\draw[->, blue] ($(distilend2) + (-0.4+5*\colwidth,0.35)$) to[bend left=45] ($(distilend2) + (-0.4+5*\colwidth,-0.05)$);
\draw[->, red] ($(distilend2) + (-0.4+6*\colwidth,0.35)$) to[bend left=45] ($(distilend2) + (-0.4+6*\colwidth,-0.05)$);

\end{tikzpicture}
\end{table*}